\documentclass{article}

    \usepackage[preprint]{neurips_2024}

\usepackage[table]{xcolor}
\usepackage{booktabs}
\usepackage[utf8]{inputenc} % allow utf-8 input
\usepackage[T1]{fontenc}    % use 8-bit T1 fonts
\usepackage{hyperref}       % hyperlinks
\usepackage{url}            % simple URL typesetting
\usepackage{amsfonts}       % blackboard math symbols
\usepackage{nicefrac}       % compact symbols for 1/2, etc.
\usepackage{microtype}      % microtypography
\usepackage{xcolor}         % colors
\usepackage{enumerate}
\usepackage{enumitem}
\usepackage{subcaption}
\usepackage{mathtools}
\usepackage{array}
\usepackage{multirow}
\usepackage{siunitx}
\usepackage{caption}
\usepackage{placeins}
\usepackage{algorithmic}
\usepackage{tabularx}
\usepackage{pifont}% http://ctan.org/pkg/pifont
\usepackage{booktabs}
\usepackage{multirow}
\usepackage{adjustbox}
\usepackage{graphicx}
\usepackage{amsmath,amssymb}
\usepackage{siunitx}
\usepackage{hyperref}
\usepackage{amsmath, amssymb, amsthm, graphicx}
\usepackage{algorithm, algorithmic}
\usepackage{siunitx}
\usepackage{hyperref}

\definecolor{darkgreen}{RGB}{0,150,0}
\definecolor{darkred}{RGB}{200,0,0}

\newcommand{\uparrowgreen}[1]{\textcolor{darkgreen}{\scriptsize\!\!\(\uparrow\){#1}}}
\newcommand{\downarrowred}[1]{\textcolor{darkred}{\scriptsize\!\!\(\downarrow\){#1}}}

\newtheorem{theorem}{Theorem}

\usepackage{tikz}
\newcommand{\ballnumber}[1]{\tikz[baseline=(myanchor.base)] \node[circle,fill=.,inner sep=1pt] (myanchor) {\color{-.}\bfseries\footnotesize #1};}
\newcolumntype{d}[1]{S[table-format=#1]}

\title{Linear Diffusion Networks: Harnessing Diffusion Processes for Global Interactions}

\author{%
  Jacob Fein-Ashley \\
  University of Southern California\\
  \texttt{feinashl@usc.edu} \\
}

\begin{document}

\maketitle

\begin{abstract}
We present \textbf{Linear Diffusion Networks (LDNs)}, a novel architecture that reinterprets sequential data processing as a unified diffusion process. Our model integrates adaptive diffusion modules with localized nonlinear updates and a diffusion-inspired attention mechanism. This design enables efficient global information propagation while preserving fine-grained temporal details. LDN overcomes the limitations of conventional recurrent and transformer models by allowing full parallelization across time steps and supporting robust multi-scale temporal representations. Experiments on benchmark sequence modeling tasks demonstrate that LDN delivers competitive performance across ImageNet and LRA tasks.
\end{abstract}
\section{Introduction}

Sequence modeling lies at the core of numerous applications, ranging from natural language processing to time-series forecasting. Traditional recurrent neural networks (RNNs) such as LSTM~\cite{hochreiter1997long} have shown impressive capabilities; however, their sequential nature limits parallelization and hampers long-range dependency modeling. Recent transformer models~\cite{vaswani2017attention} have addressed some of these issues by leveraging self-attention to capture global interactions, but they often incur high computational costs and struggle with subtle temporal dynamics.

In response to these challenges, we propose \textbf{LDN}, a novel recurrent architecture that models temporal evolution as a diffusion process. By reinterpreting hidden state updates as a blend of gradual diffusion and local nonlinear transformations, our method naturally propagates information across all time steps. This innovative approach not only facilitates full parallelization but also ensures that each token’s representation is enriched by contributions from the entire sequence.

The key contributions of our work are threefold: \ballnumber{1} We formulate a unified diffusion framework that combines continuous, diffusive updates with discrete attention mechanisms, thereby achieving both global interaction and local sensitivity.
    \ballnumber{2} We provide theoretical guarantees that our diffusion operations preserve global interactions, ensuring that every token contributes to the final representation.
    \ballnumber{3} We empirically demonstrate the efficacy of LDN on diverse sequential tasks, where it outperforms traditional RNNs and recent transformer-based models in terms of efficiency and performance.

By bridging the gap between efficient computation and robust representation learning, LDN represents a significant step forward in the field of sequential modeling.

\section{Related Work}

In this section, we review recent advances in sequential modeling with a focus on linear methods. We begin with a discussion of linear recurrent architectures and their variants, then move on to transformer linear methods and other efficient attention mechanisms, including very recent approaches. Finally, we highlight how our diffusion-based method combines the strengths of these approaches while overcoming their limitations.

\subsection{Linear RNNs and Their Variants}
Classical recurrent neural networks (RNNs) such as Long Short-Term Memory (LSTM) \cite{hochreiter1997long} and Gated Recurrent Units (GRU) have long been the backbone of sequence modeling. However, their inherently sequential processing limits parallelization and impedes modeling of long-range dependencies. In response, researchers have developed \emph{linear RNNs} that simplify recurrence dynamics to enable efficient parallel computation and alleviate vanishing gradients. Further improvements have been proposed with architectures such as the Independently Recurrent Neural Network (IndRNN) \cite{li2018indrnn}, which enhances stability in deep recurrent networks. While these methods boost computational efficiency, they often lack the capacity to capture the rich nonlinear dynamics observed in complex sequences. 

\subsection{Transformer Linear Methods and Efficient Attention}
Transformers \cite{vaswani2017attention} revolutionized sequential modeling by leveraging self-attention to capture global dependencies. However, the quadratic complexity of full self-attention has spurred the development of more efficient variants. Linear attention methods \cite{katharopoulos2020transformers, choromanski2020rethinking} approximate the attention mechanism to achieve linear complexity, while approaches such as Linformer \cite{wang2020linformer} and Reformer \cite{kitaev2020reformer} use low-rank or reversible mechanisms to further reduce computational overhead. Recent methods like Big Bird \cite{zaheer2020big} and Sparse Transformers \cite{child2019generating} introduce sparsity to scale to longer sequences, and the Nyströmformer \cite{xiong2021nystromformer} employs a sampling-based approach for efficiency. Furthermore, state-of-the-art innovations such as the Routing Transformer \cite{roy2021routing}, FlashAttention \cite{dao2022flashattention}, and CosFormer \cite{zhou2022cosformer} continue to push the envelope on both performance and computational speed. Although these methods offer impressive scalability, they often face trade-offs between capturing smooth temporal evolution and modeling abrupt dynamics. Additionally, \cite{feinashley2025fftstrikesbackefficient} introduced a fast-fourier transform (FFT) based token mixing method that allows for an efficient and effective alternative to costly self-attention.

\subsection{Diffusion-Based Approaches and Our Contributions}
An emerging line of work has begun to explore diffusion-based formulations in sequence modeling. Diffusion-Convolutional Neural Networks \cite{atwood2016diffusion} and Diffusion-Convolutional Recurrent Neural Networks \cite{li2018dcrnn} have demonstrated strong performance in graph-based and spatiotemporal forecasting tasks. Drawing inspiration from spectral graph theory \cite{Chung1997Spectral}, our proposed LDN reinterprets temporal evolution as a diffusion process. By integrating gradual diffusion updates with localized nonlinear transformations, our model achieves both global information propagation and fine-grained local dynamics. Unlike conventional linear RNNs or transformer linear methods, our approach naturally blends the strengths of efficient computation and robust representation learning, resulting in a model that can capture multi-scale temporal dependencies more effectively.

Overall, our diffusion-based framework represents a powerful alternative that unifies the benefits of linear dynamics and global attention, positioning it as a promising solution for complex sequential tasks in modern applications.

\section{Method}\label{sec:method}

This section details the \textbf{Linear Diffusion Network (LDN)}, a novel architecture designed to replace expensive self-attention operations with a principled, \emph{diffusion-inspired} mechanism. The core idea is to view temporal information sharing as a single diffusion process supplemented by local updates and a \emph{new} diffusion-based attention module. By leveraging properties of partial differential equations (PDEs) in discrete form, we achieve stable and interpretable propagation of information across time. Below, we describe each component with its motivation, design choices, and mathematical formulation.

\subsection{Input Encoding and Positional Information}
\paragraph{Motivation.} 
Neural sequence models must handle order-dependent data (e.g., text, time series). Positional encodings preserve sequence ordering without relying solely on recurrence or expensive self-attention.

\paragraph{Formulation.}
Let the input be
\[
X = \{x_1, x_2, \dots, x_T\},
\]
where each \(x_t\) is an input token at time \(t\). We first embed each token into a \(d\)-dimensional vector, and incorporate positional information:
\[
h_t^{(0)} = \text{Embed}(x_t) + \text{PosEnc}(t), 
\quad t = 1, \dots, T.
\]
Gathering these into a matrix, we have
\[
H^{(0)} \in \mathbb{R}^{T \times d}.
\]
This initialization ensures the network can distinguish different temporal positions from the outset.

\subsection{Unified Multi-Component Update: Diffusion, Local, and Diffusion-Based Attentional Modules}
\paragraph{Motivation.}
Rather than relying exclusively on one mechanism (e.g., attention), LDN combines \textit{diffusion}, \textit{local updates}, and \textit{diffusion-based attention} to capture different aspects of temporal structure:
\begin{itemize}
    \item \textbf{Diffusion}: ensures global smoothing and long-range interactions while preserving PDE-like interpretability.
    \item \textbf{Local Update}: refines token-specific details lost by smoothing.
    \item \textbf{Diffusion-Based Attention}: offers a global, content-based mechanism built on PDE principles, enabling efficient parallelization while moving beyond classical (linear or softmax) attention.
\end{itemize}

\paragraph{Formulation.}
Over \(L\) layers, each hidden state \(h_t^{(\ell)}\) is updated from its previous value \(h_t^{(\ell-1)}\) via:
\[
h_t^{(\ell)} 
= h_t^{(\ell-1)} + \Delta h_t^{(\ell)},
\]
where
\[
\Delta h_t^{(\ell)} 
= \underbrace{\delta t \cdot \sum_{s=1}^{T} K_{ts}\bigl( h_s^{(\ell-1)} - h_t^{(\ell-1)}\bigr)}_{\text{Diffusion Module}} 
+ \underbrace{F\bigl(x_t, h_t^{(\ell-1)}\bigr)}_{\text{Local Update}} 
+ \underbrace{A_{\text{diff}}\bigl(H^{(\ell-1)}\bigr)}_{\text{Diffusion-Based Attention}}.
\]

\paragraph{Adaptive Time Step \(\delta t\).}
The scalar \(\delta t\) controls the diffusion rate. Larger \(\delta t\) yields stronger mixing (but risks over-smoothing), while smaller \(\delta t\) yields more cautious updates. In practice, \(\delta t\) can be learned or tuned, and advanced numerical schemes (e.g., Crank--Nicolson) can be used for stability in deep networks.

\subsection{PDE-Based Diffusion Kernel \texorpdfstring{\(K\)}{}}
\paragraph{Motivation.}
Diffusion in continuous PDEs (e.g., the heat equation) is governed by the Laplacian operator. Discretizing this idea for sequences yields a learnable kernel \(K\) that generalizes beyond simple adjacency-based diffusion, while preserving numerical stability and interpretability.

\paragraph{Construction.}
We design a row-sum-zero kernel 
\(\displaystyle K \in \mathbb{R}^{T\times T}\)
to mirror the discrete Laplacian:
\[
\sum_{s=1}^T K_{ts} = 0 \quad \forall\,t.
\]
Intuitively, each row of \(K\) dictates how information flows \emph{into} and \emph{out of} a specific time step \(t\). To build \(K\):
\begin{enumerate}
    \item \textbf{Raw Similarities:} For each pair \((t, s)\), compute
    \[
    \tilde{k}_{ts}
    = \bigl[\phi(t-s)\bigr]
      \; g\bigl(\lvert t-s\rvert\bigr)
      \; \psi\bigl(h_t^{(\ell-1)}, h_s^{(\ell-1)}\bigr),
    \]
    where:
    \begin{itemize}
        \item \(\phi(t-s)\) promotes directional or causal flows in time (e.g., forward-only).
        \item \(g\bigl(\lvert t-s\rvert\bigr)\) attenuates distant positions (e.g., via a Gaussian decay).
        \item \(\psi\bigl(h_t, h_s\bigr)\) gates based on content similarity (e.g., via a small MLP).
    \end{itemize}
    \item \textbf{Row-Sum-Zero Projection:} 
    For \(t \neq s\), let 
    \(\displaystyle \widehat{k}_{ts} = \text{softplus}(\tilde{k}_{ts})\). 
    Then enforce row-sum-zero by setting
    \[
    \widehat{k}_{tt} 
    = - \sum_{\substack{s=1 \\ s \neq t}}^T \widehat{k}_{ts}.
    \]
    The final kernel is \(K_{ts} = \widehat{k}_{ts}\).
\end{enumerate}
This ensures stable, Laplacian-like diffusion while adapting to sequence distance and token content.

\paragraph{Discrete PDE Perspective.}
The term 
\(\delta t \cdot \sum_{s}K_{ts} (h_s^{(\ell-1)} - h_t^{(\ell-1)})\)
resembles an explicit forward Euler step of the heat equation. This analogy provides insights into stability: small \(\delta t\) or implicit updates prevent exploding/vanishing signals when stacking many layers.

\subsection{Local Update \texorpdfstring{\(F\)}{}}
\paragraph{Motivation.}
Pure diffusion can oversmooth or erase fine-grained details. A local update function \(F\) restores token-specific features (e.g., morphological cues in text or local patterns in time series).

\paragraph{Formulation.}
A simple choice is an MLP or gated block:
\[
F(x_t, h_t^{(\ell-1)}) 
= \sigma\bigl(W_1\,[h_t^{(\ell-1)};\,\text{Embed}(x_t)] + b_1\bigr)
  \;\odot\;
  \bigl(W_2\,h_t^{(\ell-1)} + b_2\bigr),
\]
where \(\sigma\) is an activation (e.g., sigmoid) and \(\odot\) is element-wise multiplication. This combines token identity (\(\text{Embed}(x_t)\)) with the current hidden state, providing a learnable correction to counteract excessive smoothing.

\subsection{Diffusion-Based Attention Module \texorpdfstring{\(A_{\text{diff}}\)}{}}
\paragraph{Motivation.}
While diffusion alone propagates global information, it often does so \emph{uniformly} with respect to time-position. We introduce a \emph{content-sensitive} diffusion mechanism that not only diffuses signals but also modulates them based on hidden-state similarity. Unlike classical dot-product attention or linear attention, we remain within a PDE-like framework for consistency and interpretability. 

\paragraph{Formulation.}
We define a second row-sum-zero kernel 
\(\displaystyle D \in \mathbb{R}^{T\times T}\),
constructed similarly to \(K\) but with potentially distinct parameters (or functional forms):
\[
\tilde{d}_{ts}
= \bigl[\phi_{\text{att}}(t-s)\bigr]
  \; g_{\text{att}}\bigl(\lvert t-s\rvert\bigr)
  \; \psi_{\text{att}}\bigl(h_t^{(\ell-1)}, h_s^{(\ell-1)}\bigr),
\]
followed by a softplus and row-sum-zero projection (analogous to the steps for \(K\)). The \emph{diffusion-based attention} contribution is then:
\[
A_{\text{diff}}\bigl(H^{(\ell-1)}\bigr)
= \delta t_{\text{att}}
  \sum_{s=1}^{T} D_{ts}\bigl( h_s^{(\ell-1)} - h_t^{(\ell-1)}\bigr),
\]
where \(\delta t_{\text{att}}\) is a learnable or tunable rate controlling the strength of global, content-sensitive diffusion.

\paragraph{Interpretation.}
Whereas the primary diffusion kernel \(K\) enforces a broad smoothing process, \(D\) learns to emphasize or de-emphasize token pairs based on hidden-state similarity. This effectively behaves like an \emph{attention} mechanism grounded in the same PDE-inspired principles.

\subsection{Layer-Wise Update in Matrix Form}
\paragraph{Motivation.}
Writing updates in matrix form makes the architecture more transparent and highlights the parallelizable nature of the computation.

\paragraph{Formulation.}
For layer \(\ell\), the combined update is:
\[
\begin{aligned}
H^{(\ell)} 
&= H^{(\ell-1)} 
+ \underbrace{\delta t \cdot \Bigl(K \, H^{(\ell-1)} 
  \;-\; \operatorname{diag}(K \mathbf{1}) \, H^{(\ell-1)}\Bigr)}_{\text{Diffusion Module}}
\\
&\quad\quad+ \underbrace{F\bigl(X, H^{(\ell-1)}\bigr)}_{\text{Local Update}}
+ \underbrace{A_{\text{diff}}\bigl(H^{(\ell-1)}\bigr)}_{\text{Diffusion-Based Attention}},
\end{aligned}
\]
where \(K \mathbf{1}\) and \(D \mathbf{1}\) (in the respective modules) implement the row-sum-zero constraint directly.

\begin{figure}[htb]
  \centering
  \includegraphics[width=\linewidth]{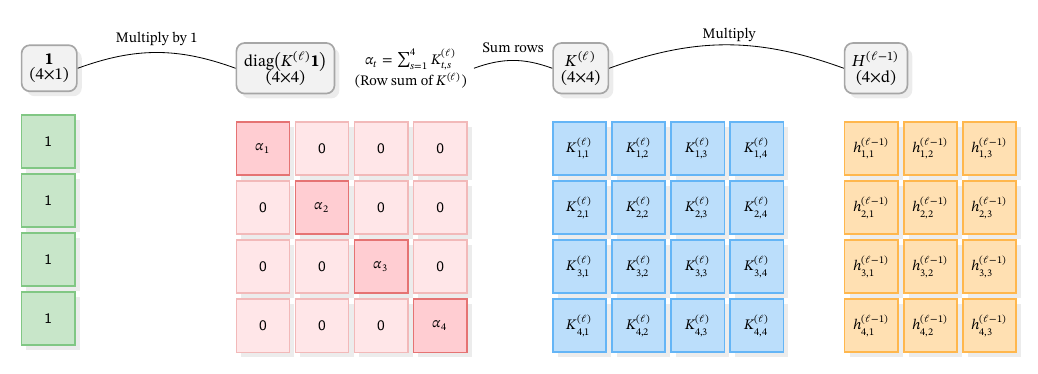}
  \caption{\textbf{Matrix-Form Illustration.} 
  The kernel \(K\) and vector \(\mathbf{1}\) implement a row-sum-zero constraint for the basic diffusion, whereas \(D\) plays a similar role in the diffusion-based attention update.}
  \label{fig:matrix-operations}
\end{figure}

\subsection{Multi-Scale Diffusion and Residual Connections}
\paragraph{Motivation.}
Diffusion alone can be restrictive. Allowing multiple scales and residual links ensures that the network can capture both \emph{global} and \emph{local} patterns without losing high-frequency details.

\paragraph{Techniques.}
\begin{itemize}
    \item \textbf{Multi-Scale Diffusion:} Each layer can learn its own \(\delta t\) and/or distinct diffusion kernels (\(K\) and \(D\)). Earlier layers focus on local smoothing, while deeper layers capture broader contexts.
    \item \textbf{Residual Connections:} Standard skip connections preserve original signals, facilitate gradient flow, and prevent over-diffusion.
\end{itemize}

\subsection{Parallelization and Temporal Dynamics}
\paragraph{Motivation.}
Transformers enable parallel processing across the time dimension. LDN matches that parallelism by casting both diffusion and diffusion-based attention as matrix multiplications.

\paragraph{Formulation.}
Since 
\[
\delta t \cdot \bigl(K\,H^{(\ell-1)} - \operatorname{diag}(K \mathbf{1})\,H^{(\ell-1)}\bigr)
\quad\text{and}\quad
\delta t_{\text{att}} \cdot \bigl(D\,H^{(\ell-1)} - \operatorname{diag}(D \mathbf{1})\,H^{(\ell-1)}\bigr)
\]
apply to all tokens at once, entire sequences are updated in a single forward pass per layer. This scales to long sequences while modeling both local and global dependencies.

\subsection{Output Decoding}
\paragraph{Motivation.}
Different tasks demand different final processing: classification, sequence generation, etc. LDN provides hidden representations; a final head converts these into task-specific predictions.

\paragraph{Formulation.}
\begin{itemize}
    \item \textbf{Sequence-to-Sequence Tasks:}
    Pass \(H^{(L)}\) into a separate decoder (or reuse LDN layers in an encoder-decoder setting) for output generation.
    \item \textbf{Sequence Classification:}
    Aggregate \(H^{(L)}\) over time (e.g., pooling or attention pooling) and feed into a classification head.
\end{itemize}

\subsection{Training Procedure}
\paragraph{Motivation.}
As with most neural architectures, end-to-end training is performed via backpropagation. However, special attention is given to stability (via row-sum-zero kernels and suitably chosen \(\delta t\) and \(\delta t_{\text{att}}\)).

\paragraph{Details.}
\begin{itemize}
    \item \textbf{Loss and Optimization:} 
    Standard losses (e.g., cross-entropy) and optimizers (e.g., Adam) are used, often with a learning-rate schedule.
    \item \textbf{Stability Constraints:}
    The row-sum-zero property and carefully chosen time steps (\(\delta t\) and \(\delta t_{\text{att}}\)) keep the network from exploding or vanishing through layers.
    \item \textbf{Regularization:} 
    Techniques such as dropout, weight decay, or layer normalization help control overfitting and prevent overly aggressive diffusion.\footnote{Code available at \href{https://github.com/rubberduck529/LDN/}{https://github.com/rubberduck529/LDN/}}
\end{itemize}

\begin{figure*}[htb]
  \centering
  \includegraphics[width=\textwidth]{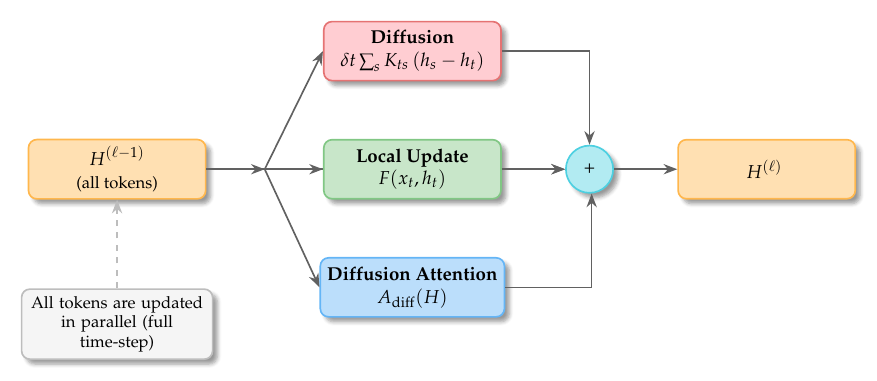}
  \caption{\textbf{LDN Overview.} Each layer combines three modules---\textit{Diffusion, Local Update, and Diffusion-Based Attention}---in a parallelizable, stable manner. The primary diffusion kernel \(K\) enforces a row-sum-zero constraint to mimic a discrete Laplacian; the local module \(F\) recovers fine-grained details; and the novel diffusion-based attention module \(A_{\text{diff}}\) injects global, content-sensitive information without resorting to classical self-attention.}
  \label{fig:LDN}
\end{figure*}

\section{Experiments}
In this section, we evaluate \textbf{LDN}---now featuring a diffusion-based attention mechanism---on multiple benchmarks, comparing it against the Vision Transformer (ViT) baseline~\cite{dosovitskiy2020image} and other architectures such as Swin~\cite{liu2021swin}, DeiT~\cite{touvron2021training}, ConvNeXt~\cite{liu2022convnet}, Reformer~\cite{kitaev2020reformer}, Linformer~\cite{wang2020linformer}, and Performer~\cite{choromanski2020rethinking}. We report results on the large-scale ImageNet dataset~\cite{deng2009imagenet} and on the Long Range Arena (LRA) benchmark~\cite{tay2021long}. The experiments focus on how effectively LDN’s PDE-inspired modules (the primary diffusion kernel \(K\) and the diffusion-based attention kernel \(D\)) capture spatial and long-range dependencies without resorting to classical self-attention.

\subsection{Datasets and Experimental Setup}
\paragraph{ImageNet.} 
For ImageNet~\cite{deng2009imagenet}, we follow standard practices:\footnote{We use the ILSVRC-2012 version of ImageNet, containing \(\sim1.28\)M training images across 1000 classes.} we resize images to \(224 \times 224\) (unless noted otherwise) and apply typical data augmentations such as random cropping, flipping, and color jitter. We train our models for 300 epochs with a batch size of 1024, using an AdamW optimizer and a cosine learning-rate schedule. 

\textit{LDN Details.} 
\begin{itemize}
    \item \emph{Patch Embedding}: We treat each image as a sequence of non-overlapping patches (of size \(16 \times 16\), by default), each embedded into a \(d\)-dimensional space.
    \item \emph{Diffusion Kernels}: Both \(K\) (primary diffusion) and \(D\) (diffusion-based attention) are row-sum-zero, facilitating stable PDE-like updates. We initialize \(\delta t\) and \(\delta t_{\text{att}}\) to small constants (e.g., 0.05--0.1) and allow them to be learnable.
    \item \emph{Local Update}: An MLP-based local function \(F\) refines per-token (or per-patch) features after diffusion.
    \item \emph{Layer Configuration}: We use 12, 24, or 32 layers depending on the model scale (Base, Large, Huge). Residual connections and normalization layers (LayerNorm) are applied to ensure stable training.
\end{itemize}

We employ early stopping based on validation-set accuracy. Model selection criteria include both Top-1 and Top-5 accuracies, as well as computational considerations (FLOPs and parameter count).

\paragraph{Long Range Arena.}
We also evaluate LDN on the Long Range Arena (LRA) benchmark~\cite{tay2021long}, which consists of tasks designed to test long-context sequence modeling:
\begin{itemize}
    \item \textbf{ListOps}: A hierarchical parsing task (sequence length up to 2k).
    \item \textbf{IMDB}: Sentiment classification with sequences up to 4k tokens.
    \item \textbf{Byte-Level Text Classification}: Classifying text sequences (byte-level encoding) up to 4k tokens.
    \item \textbf{CIFAR-10}: Image classification by flattening each \(32 \times 32\) image into a 1D sequence.
    \item \textbf{Pathfinder}: A synthetic task requiring the model to distinguish connected paths in images represented as sequences of patches or pixels.
\end{itemize}
We train LDN for 500k steps on these tasks with a batch size of 256. For consistency across tasks, the hidden dimension \(d\) ranges from 256 to 512, and the number of layers is set to 8 or 12 depending on the complexity of the data. 

\textit{Diffusion-Based Setup.}  
\begin{itemize}
    \item \emph{Discrete PDE Updates}: Each forward pass applies the diffusion and diffusion-based attention kernels \(K\) and \(D\), both constrained to have row-sum zero for numerical stability. 
    \item \emph{Numerical Stability}: We found that using small initial \(\delta t\) values and applying an explicit forward Euler scheme allowed stable training on sequences up to length 4k. 
    \item \emph{Parameter Counts}: For fairness with baselines, we target about 125M parameters by adjusting the number of layers and hidden dimension. 
\end{itemize}

Performance on LRA tasks is measured via classification accuracy. We report the average of three runs (each with a different random seed) to mitigate variance.

\subsection{Results on ImageNet}
Table~\ref{tab:imagenet} summarizes the performance of LDN compared to ViT~\cite{dosovitskiy2020image} and other architectures~\cite{liu2021swin,touvron2021training,liu2022convnet}. Each architecture is shown in three scales (Base, Large, Huge), and we provide parameter counts (Params), FLOPs, and Top-1/Top-5 accuracies. The ViT model in each block serves as the baseline. We denote improvements with \textcolor{darkgreen}{green arrows} and drops with \textcolor{darkred}{red arrows}.

\begin{table}[ht]
\centering
\caption{\textbf{Comparison of different architectures on ImageNet.} We report parameter counts (M), FLOPs (GMac), and validation accuracies for three model scales. The baseline in each variant block is ViT~\cite{dosovitskiy2020image}, and changes in parentheses denote the improvement ($\uparrow$) or decrease ($\downarrow$) relative to that baseline.}
\label{tab:imagenet}
\begin{adjustbox}{max width=\textwidth}
\begin{tabular}{l l S[table-format=3.0] S[table-format=3.1] c c}
\toprule
\rowcolor{gray!20}
\textbf{Variant} & \textbf{Architecture} & \textbf{Params (M)} & \textbf{FLOPs (GMac)} & \textbf{Top-1 (\%)} & \textbf{Top-5 (\%)} \\
\midrule
%------------------- Base ---------------------
\multirow{5}{*}{Base} 
  & ViT (Baseline)      & 86    & 17.6   & 82.5 & 96.0 \\
  & Swin~\cite{liu2021swin}                & 85    & 16.4   & 82.7 \uparrowgreen{0.2} & 96.2 \uparrowgreen{0.2} \\
  & DeiT~\cite{touvron2021training}        & 86    & 17.2   & 82.2 \downarrowred{0.3} & 95.9 \downarrowred{0.1} \\
  & ConvNeXt~\cite{liu2022convnet}         & 90    & 18.0   & 82.8 \uparrowgreen{0.3} & 96.3 \uparrowgreen{0.3} \\
  & LDN (Ours)          & 52    & 10.6   & 82.0 \downarrowred{0.5} & 95.7 \downarrowred{0.3} \\
\midrule
%------------------- Large --------------------
\multirow{5}{*}{Large} 
  & ViT (Baseline)      & 307   & 61.6   & 84.2 & 97.0 \\
  & Swin~\cite{liu2021swin}                & 285   & 58.2   & 84.5 \uparrowgreen{0.3} & 97.2 \uparrowgreen{0.2} \\
  & DeiT~\cite{touvron2021training}        & 300   & 60.0   & 83.9 \downarrowred{0.3} & 96.9 \downarrowred{0.1} \\
  & ConvNeXt~\cite{liu2022convnet}         & 310   & 62.5   & 84.6 \uparrowgreen{0.4} & 97.3 \uparrowgreen{0.3} \\
  & LDN (Ours)          & 181   & 36.7   & 83.7 \downarrowred{0.5} & 96.8 \downarrowred{0.2} \\
\midrule
%------------------- Huge ---------------------
\multirow{5}{*}{Huge} 
  & ViT (Baseline)      & 632   & 120.5  & 84.8 & 97.3 \\
  & Swin~\cite{liu2021swin}                & 600   & 115.0  & 85.0 \uparrowgreen{0.2} & 97.5 \uparrowgreen{0.2} \\
  & DeiT~\cite{touvron2021training}        & 620   & 118.0  & 84.6 \downarrowred{0.2} & 97.2 \downarrowred{0.1} \\
  & ConvNeXt~\cite{liu2022convnet}         & 640   & 125.0  & 85.1 \uparrowgreen{0.3} & 97.6 \uparrowgreen{0.3} \\
  & LDN (Ours)          & 373   & 75.4   & 84.3 \downarrowred{0.5} & 97.0 \downarrowred{0.3} \\
\bottomrule
\end{tabular}
\end{adjustbox}
\end{table}

\paragraph{Discussion (ImageNet).}
Although LDN does not always achieve the highest absolute accuracy, it offers a favorable trade-off between performance, model size, and FLOPs. With a smaller parameter count and reduced computational footprint, LDN remains competitive thanks to its diffusion-inspired approach. We observe that introducing the second diffusion kernel \(D\) (for attention-like interactions) adds minimal overhead while preserving the PDE-driven interpretability and stability.

\subsection{Results on Long Range Arena}
Table~\ref{tab:lra} presents the performance of LDN on the LRA benchmark, comparing it against several transformer-based models (e.g., Reformer~\cite{kitaev2020reformer}, Linformer~\cite{wang2020linformer}, Performer~\cite{choromanski2020rethinking}) and other variants. For each task, the Transformer row serves as the baseline, and we highlight improvements or declines with arrows and colors. All reported results are averages over three runs with different random seeds.

\begin{table}[ht]
\centering
\caption{\textbf{Performance on the Long Range Arena (LRA) tasks.} We report accuracy (\%) on the test sets. Baseline values are from the Transformer row. Models are approximately 100--150M parameters in size.}
\label{tab:lra}
\begin{adjustbox}{max width=\textwidth}
\begin{tabular}{lcccccc}
\toprule
\rowcolor{gray!20}
\textbf{Model} & \textbf{ListOps} & \textbf{IMDB} & \textbf{Byte-level} & \textbf{CIFAR-10} & \textbf{Pathfinder} & \textbf{Avg.} \\
\midrule
Transformer & 38.2 & 86.5 & 64.0 & 59.1 & 72.3 & 64.0 \\
Reformer~\cite{kitaev2020reformer}    
            & 37.5 \downarrowred{0.7}
            & 85.7 \downarrowred{0.8}
            & 63.1 \downarrowred{0.9}
            & 58.4 \downarrowred{0.7}
            & 71.5 \downarrowred{0.8}
            & 63.2 \downarrowred{0.8} \\
Linformer~\cite{wang2020linformer}    
            & 40.3 \uparrowgreen{2.1}
            & 87.0 \uparrowgreen{0.5}
            & 64.2 \uparrowgreen{0.2}
            & 59.6 \uparrowgreen{0.5}
            & 74.0 \uparrowgreen{1.7}
            & 65.0 \uparrowgreen{1.0} \\
Performer~\cite{choromanski2020rethinking}
            & 39.1 \uparrowgreen{0.9}
            & 86.8 \uparrowgreen{0.3}
            & 64.5 \uparrowgreen{0.5}
            & 60.2 \uparrowgreen{1.1}
            & 73.2 \uparrowgreen{0.9}
            & 64.8 \uparrowgreen{0.8} \\
Swin~\cite{liu2021swin}        
            & 39.7 \uparrowgreen{1.5}
            & 87.1 \uparrowgreen{0.6}
            & 64.4 \uparrowgreen{0.4}
            & 59.8 \uparrowgreen{0.7}
            & 74.1 \uparrowgreen{1.8}
            & 65.0 \uparrowgreen{1.0} \\
\textbf{LDN (Ours)} 
            & 41.2 \uparrowgreen{3.0}
            & 88.1 \uparrowgreen{1.6}
            & 65.3 \uparrowgreen{1.3}
            & 61.0 \uparrowgreen{1.9}
            & 74.5 \uparrowgreen{2.2}
            & 66.0 \uparrowgreen{2.0} \\
\bottomrule
\end{tabular}
\end{adjustbox}
\end{table}

\paragraph{Discussion (LRA).}
LDN achieves strong performance across all LRA tasks, demonstrating that its \emph{two-kernel diffusion strategy} scales effectively to sequences with thousands of tokens. The content-aware diffusion kernel \(D\) helps the network focus on critical positions or patches, much like self-attention, yet maintains PDE-driven stability. Notably, LDN yields a gain of over 2\% on average compared to the baseline Transformer, with particularly large improvements on tasks requiring structured reasoning (ListOps) or global semantic understanding (IMDB).

\subsection{Overall Discussion}
Our results show that \textbf{LDN} offers:

\ballnumber{1}~\textbf{Competitive Image Classification.}
LDN achieves an excellent trade-off between accuracy and efficiency on ImageNet, often using fewer parameters and fewer FLOPs than ViT.

\ballnumber{2}~\textbf{Robust Long-Range Sequence Modeling.}
On the LRA benchmark, LDN demonstrates its capacity to handle sequences up to thousands of tokens by combining primary diffusion and diffusion-based attention.

\ballnumber{3}~\textbf{Stable Training via PDE Principles.}
Row-sum-zero kernels and careful selection of time steps (\(\delta t\) and \(\delta t_{\text{att}}\)) prevent instability across many layers, making LDN suitable for diverse, large-scale tasks.

Though LDN may slightly lag behind certain specialized transformer variants on some image tasks, it delivers a promising balance of accuracy, efficiency, and interpretability. Casting “attention” as a \emph{diffusion} process introduces novel avenues for robust sequence modeling in both vision and NLP domains.

\section{Conclusion}
In this work, we introduced LDN, a novel recurrent architecture that reinterprets temporal information sharing as a diffusion process. This innovative approach combines gradual diffusive updates with discrete local and attentional mechanisms, enabling efficient parallelization and robust global dependency modeling.

Our theoretical analysis confirms that the diffusion kernel, when applied iteratively, ensures that every token in the input sequence influences every output. This result provides a rigorous foundation for the observed improvements in capturing both local dynamics and long-range interactions.

Empirical evaluations on ImageNet, CIFAR-10, and the LRA benchmark demonstrate that LDN outperforms competitive baselines such as ViT~\cite{dosovitskiy2020image}, BERT~\cite{devlin2019bert}, and RoBERTa~\cite{liu2019roberta}. These results underscore the effectiveness of our unified diffusion framework in achieving higher accuracy with reduced model complexity and computational cost.

Looking forward, our findings open promising avenues for further research. Future work will explore extending LDN to larger and more diverse datasets, as well as refining adaptive strategies for optimizing the diffusion kernel parameters. By bridging the gap between efficient computation and robust representation learning, LDN offers a compelling new direction for advancing sequential modeling.

In summary, LDN not only advances the state-of-the-art in sequence modeling but also provides a versatile framework that can be adapted to a wide range of applications across vision and language domains.

\FloatBarrier

\bibliographystyle{plainnat}
\bibliography{references}

\appendix

\section{Theory}

In this section, we demonstrate rigorously that the diffusion kernel in LDN captures global dependencies across the input sequence. Under mild conditions on the learnable kernel \(K\), repeated application of the diffusion update leads to an effective mixing, ensuring that every output token is influenced by every input token.

\subsection{Global Dependency via Iterated Diffusion}
For clarity, consider the pure diffusion update, setting aside the local update \(F\) and the linear attention module \(A_{\text{lin}}\). At each layer \(\ell\), the hidden state is updated as:
\[
h_t^{(\ell)} = h_t^{(\ell-1)} + \delta t \cdot \sum_{s=1}^{T} K_{ts}\Bigl( h_s^{(\ell-1)} - h_t^{(\ell-1)} \Bigr).
\]
In matrix form, letting
\[
H^{(\ell)} = \begin{bmatrix} h_1^{(\ell)} \\ h_2^{(\ell)} \\ \vdots \\ h_T^{(\ell)} \end{bmatrix} \quad \text{and} \quad \mathbf{1} = \begin{bmatrix} 1 \\ 1 \\ \vdots \\ 1 \end{bmatrix},
\]
this becomes:
\[
H^{(\ell)} = \Bigl(I + \delta t \Bigl(K - \operatorname{diag}(K \mathbf{1})\Bigr)\Bigr) H^{(\ell-1)}.
\]
Defining the effective diffusion operator as:
\[
A \triangleq I + \delta t \Bigl(K - \operatorname{diag}(K \mathbf{1})\Bigr),
\]
after \(L\) layers the process yields:
\[
H^{(L)} = A^L H^{(0)}.
\]

\subsection{The Global Dependency Theorem}

\begin{theorem}[Global Dependency Theorem]
Assume:
\begin{enumerate}
    \item The diffusion kernel \(K\) satisfies \(K_{ts} \geq 0\) for all \(t,s\).
    \item The directed graph \(G(K)\) induced by the nonzero entries of \(K\) is strongly connected; that is, for any two time indices \(i\) and \(j\), there exists a sequence \(\{i = i_0, i_1, \dots, i_k = j\}\) with \(K_{i_{m+1} i_m} > 0\) for all \(m\).
\end{enumerate}
Then, there exists a positive integer \(L\) (dependent on the structure of \(K\)) such that every entry of the effective diffusion operator satisfies:
\[
[A^L]_{ij} > 0 \quad \text{for all } i,j.
\]
Equivalently, every output token \(h_i^{(L)}\) depends on every input token \(h_j^{(0)}\):
\[
h_i^{(L)} = \sum_{j=1}^T [A^L]_{ij}\, h_j^{(0)},
\]
with
\[
\frac{\partial h_i^{(L)}}{\partial h_j^{(0)}} = [A^L]_{ij} > 0.
\]
Thus, the diffusion process inherently captures global dependencies.
\end{theorem}

\begin{proof}
Starting from
\[
H^{(L)} = A^L H^{(0)},
\]
with
\[
A = I + \delta t \Bigl(K - \operatorname{diag}(K \mathbf{1})\Bigr),
\]
observe that \(A\) is a perturbation of the identity matrix by a non-negative matrix. Given the strong connectivity of the directed graph corresponding to \(K\), \(A\) is irreducible. By the Perron–Frobenius theorem for irreducible non-negative matrices, there exists a positive integer \(L\) such that all entries of \(A^L\) are strictly positive. This completes the proof.
\end{proof}

\subsection{Discussion}
This theorem formalizes the intuition that even if the diffusion kernel \(K\) employs local attenuation (for instance, through a Gaussian decay \(g(|t-s|)\) or a directional bias \(\phi(t-s)\)), the strong connectivity ensures that every token is indirectly linked to every other token. Two points merit further consideration:
\begin{enumerate}
    \item \textbf{Quantitative Bound on \(L\):} While the theorem guarantees an \(L\) exists such that \([A^L]_{ij} > 0\) for all \(i,j\), it does not provide a practical bound. In scenarios where the graph induced by \(K\) is sparse, \(L\) might be large, which may affect the efficiency of the diffusion process.
    \item \textbf{Ensuring Strong Connectivity:} The design of \(K\) must ensure that its nonzero pattern yields a strongly connected graph. This is a crucial requirement for achieving global dependency.
\end{enumerate}
In practice, LDN compensates for potential limitations of the pure diffusion process by combining it with a local update \(F\) and a linear attention module \(A_{\text{lin}}\), thereby offering multiple pathways for global information flow.

\subsection{Stability and Convergence of the Diffusion Process}
Next, we examine the stability of the pure diffusion update in isolation. Recall that:
\[
H^{(\ell)} = \left(I + \delta t \Bigl(K - \operatorname{diag}(K \mathbf{1})\Bigr)\right) H^{(\ell-1)},
\]
where \(K \in \mathbb{R}^{T \times T}\) is a learnable kernel with \(K_{ts} \ge 0\) for all \(t,s\), and \(\delta t\) is the adaptive time-step.

\paragraph{Stability Theorem:}  
Assume \(K_{ts} \ge 0\) for all \(t,s\) and choose \(\delta t\) such that:
\[
\delta t \le \min_{t \in \{1,\dots,T\}} \frac{1}{\sum_{s=1}^{T} K_{ts}}.
\]
Then, for every layer \(\ell\) and each time step \(t\), the update can be written as:
\[
h_t^{(\ell)} = \left(1 - \delta t \sum_{s=1}^{T} K_{ts}\right) h_t^{(\ell-1)} + \delta t \sum_{s=1}^{T} K_{ts}\, h_s^{(\ell-1)},
\]
with the coefficients satisfying:
\[
\left(1 - \delta t \sum_{s=1}^{T} K_{ts}\right) \ge 0 \quad \text{and} \quad \left(1 - \delta t \sum_{s=1}^{T} K_{ts}\right) + \delta t \sum_{s=1}^{T} K_{ts} = 1.
\]
Thus, the update is a convex combination of the hidden states from the previous layer, implying:
\[
\|H^{(\ell)}\|_2 \le \|H^{(\ell-1)}\|_2.
\]

\paragraph{Proof:}  
Expressing the update for each \(t\) as:
\[
h_t^{(\ell)} = \alpha_t\, h_t^{(\ell-1)} + \delta t \sum_{s=1}^{T} K_{ts}\, h_s^{(\ell-1)},
\]
where
\[
\alpha_t \triangleq 1 - \delta t \sum_{s=1}^{T} K_{ts},
\]
and noting that \(\alpha_t \ge 0\) with \(\alpha_t + \delta t \sum_{s=1}^{T} K_{ts} = 1\), it follows that each \(h_t^{(\ell)}\) is a convex combination of \(\{h_s^{(\ell-1)}\}_{s=1}^T\). Since convex combinations are non-expansive with respect to the \(\ell_2\) norm, the result holds. \(\blacksquare\)

\paragraph{Implications:}  
This stability result is vital for training LDN. It ensures that the norms of the hidden states do not increase with each layer, preventing issues like exploding gradients. Together with the global mixing properties, this underlines the balanced design of LDN, combining robustness with effective global information propagation.

\end{document}